\documentclass[letterpaper, 10pt, conference]{ieeeconf}
\usepackage{times}

\IEEEoverridecommandlockouts                              

\usepackage{mathtools}
\usepackage{makeidx}   
\usepackage{graphicx}  
\usepackage{multicol}  
\usepackage[bottom]{footmisc}

\usepackage[font=small]{caption}
\usepackage{amsmath,amssymb}
\usepackage{calc}
\usepackage{tikz}
\usepackage{xfrac}
\usepackage{array}
\usepackage{bm}

\usepackage[ruled, vlined, linesnumbered, noend]{algorithm2e}

\usepackage{amsthm}
\usepackage{comment}
\usepackage{etoolbox}
\usepackage{url}

\usepackage[noadjust]{cite}

\makeatletter
\patchcmd{\@makecaption}
  {\scshape}
  {}
  {}
  {}
\makeatother

\usepackage{multirow}
\AtBeginEnvironment{pmatrix}{\setlength{\arraycolsep}{4pt}}

\newif\ifmargincomments 
\margincommentsfalse

\newif\ifrelaxedv  
\relaxedvtrue

\newif\ifextendedv   
\extendedvtrue

\ifmargincomments
\newcommand{\mpmargin}[2]{{\color{green}#1}\marginpar{\color{green}\raggedright\footnotesize [MP]:#2}}
\newcommand{\frmargin}[2]{{\color{brown}#1}\marginpar{\color{brown}\raggedright\footnotesize [FR]:#2}}

\newcommand{\ksmargin}[2]{{\color{cyan}#1}\marginpar{\color{cyan}\raggedright\footnotesize [KS]:#2}}
\newcommand{\ksline}[2]{{\color{blue}#1}{\em \color{blue}[KS]: #2}}
\newcommand{\frline}[2]{{\color{blue}#1}{\em \color{blue}[FR]: #2}}
\newcommand{\todo}[1]{{\color{red}{\bf TODO:} #1}}
\else
\newcommand{\mpmargin}[2]{#1}
\newcommand{\frmargin}[2]{#1}

\newcommand{\ksmargin}[2]{#1}
\newcommand{\ksline}[2]{#1}
\newcommand{\frline}[2]{#1}
\newcommand{\todo}[1]{}
\fi

\ifrelaxedv
\newcommand{\jvspace}[1]{}
\else
\renewcommand{\baselinestretch}{0.933}
\newcommand{\jvspace}[1]{\vspace{#1}}
\fi

\makeatletter
\newcommand{\pushright}[1]{\ifmeasuring@#1\else\omit\hfill$\displaystyle#1$\fi\ignorespaces}
\newcommand{\pushleft}[1]{\ifmeasuring@#1\else\omit$\displaystyle#1$\hfill\fi\ignorespaces}
\makeatother

\newtheorem{theorem}{Theorem}

\newtheorem{lemma}{Lemma}
\newtheorem{corollary}{Corollary}
\newtheorem{claim}{Claim}
\newtheorem{alg}{Algorithm}

\graphicspath{{./},{./fig/}}

\usepackage[utf8]{inputenc}

\makeatletter
\let\NAT@parse\undefined
\makeatother
\usepackage[colorlinks=true, bookmarks=true, pagebackref=true,citecolor=blue]{hyperref}

\usepackage{cleveref}

\newcommand{\cpp}{C\raise.08ex\hbox{\tt ++}\xspace}

  \def\H{\mathcal{H}}
\def\F{\mathcal{F}}  
\def\O{\mathcal{O}}  
\def\G{\mathcal{G}}  
  
\def\V{\mathcal{V}}  
 \def\R{\mathcal{R}} 
 \def\A{\mathcal{A}} 
 \def\K{\mathcal{K}} \def\E{\mathcal{E}}

 \def\dT{\mathbb{T}} 
\def\dG{\mathbb{G}}  \def\dV{\mathbb{V}}
 \def\dN{\mathbb{N}} \def\dE{\mathbb{E}}
\def\dR{\mathbb{R}}  
  
 \def\dP{\mathbb{P}}

\DeclareMathOperator*{\Ex}{\mathbb{E}}

\def\Ex{\mathbf{E}} 

\renewcommand{\leq}{\leqslant}
\renewcommand{\geq}{\geqslant}

\newcommand{\argmin}{\operatornamewithlimits{argmin}}

\def\alg{\textsc{FlowDec}\xspace}
\def\distalg{\textsc{DistributedPFSF}\xspace}

\def\sf{\textup{SF}}

\def\hom{\textsc{DistributedHomogeneous}\xspace}

\def\private{\textsc{PrivateFirst}\xspace}
\def\shared{\textsc{SharedFirst}\xspace}
\def\distprivate{\textsc{DistributedPrivateFirst}\xspace}
\def\distshared{\textsc{DistributedSharedFirst}\xspace}

\newcommand{\isep}{\mathinner {\ldotp \ldotp}}

\newcommand{\fakeparagraph}[1]{\vspace{5pt}
\noindent\textbf{#1}}

\title{\LARGE \bf{Fast Near-Optimal Heterogeneous Task Allocation \\ via Flow Decomposition}}
\author{Kiril Solovey$^1$, Saptarshi Bandyopadhyay$^2$,  Federico Rossi$^2$, Michael T. Wolf$^3$, and Marco Pavone$^1$
\thanks{$^1$K. Solovey and M. Pavone are with the Department of Aeronautics \& Astronautics, Stanford University, Stanford, CA, 94305; {\tt \{kirilsol,pavone\}@stanford.edu}.}
\thanks{$^2$S. Bandyopadhyay and F. Rossi are with the Jet Propulsion Laboratory, California Institute of Technology, Pasadena, CA, 91109; {\tt \{saptarshi.bandyopadhyay, federico.rossi\}@jpl.nasa.gov}.}
\thanks{$^3$M. T. Wolf was with the Jet Propulsion Laboratory, California Institute of Technology, Pasadena, CA, 91109 for this work. He is now with Amazon Robotics AI. {\tt wolf@robotics.caltech.edu}.}}
\begin{document}

\maketitle

\begin{abstract}
Multi-robot systems are uniquely well-suited to performing
complex tasks such as patrolling and tracking, information
gathering, and pick-up and delivery problems, offering
significantly higher performance than single-robot systems.
A fundamental building block in most multi-robot systems is
task allocation: assigning robots to tasks (e.g.,
patrolling an area, or servicing a transportation request)
as they appear based on the robots' states to maximize
reward. In many practical situations, the allocation must
account for heterogeneous capabilities (e.g., availability
of appropriate sensors or actuators) to ensure the
feasibility of execution, and to promote a higher reward,
over a long time horizon. To this end, we present the
\alg algorithm for efficient heterogeneous
task-allocation, and show that it achieves an approximation factor of at
least $\bm{1/2}$ of the optimal reward. Our approach decomposes
the heterogeneous problem into several homogeneous
subproblems that can be solved efficiently using min-cost
flow. Through simulation experiments, we show that our
algorithm is faster by several orders of magnitude than a
MILP approach. 
\end{abstract}

\section{Introduction}\label{sec:introduction}
A central problem in many multi-robot applications, including patrolling, information gathering, and pick-up and delivery problems,
is \emph{task allocation}: that is, to assign robots to outstanding, spatially-distributed tasks (e.g. patrolling an area or servicing a transportation request) based on the robots' states (e.g., position and power-level) and potentially heterogeneous capabilities (e.g., availability of appropriate sensors or actuators), and accounting not only for current tasks but also for the likelihood that future tasks will appear.

In this paper we consider a task-allocation setting where mobile heterogeneous robots need be assigned to time-varying sets of tasks, or, equivalently, rewards. In our formulation, robots are divided into homogeneous fleets based on their ability to collect \emph{private} reward sets (where each fleet is associated with a unique private set). There is also a \emph{shared} reward set that can be collected by robots in any fleet. A robot is rewarded for a specific reward set if (i) it resides in the spatial vicinity of the reward, (ii) the reward is either shared or private to the robot's specific fleet, and (iii) no other robot is assigned to this reward.

A number of problems of interest fall in this setting, including object tracking, intruder following and imaging a scientific phenomenon. We are particularly motivated by data-gathering and agile science applications for planetary science where multiple spacecraft detect events of interest and then perform follow-up scientific observations. A number of concepts have been proposed in this setting, including multi-spacecraft constellations to study the Martian atmosphere and the dust cycle~\cite{lillis2020mars}, networks of balloons to detect seismic and volcanic events on Venus \cite{KrishnamoortyKomjathyEtAl2020,KrishnamoorthyLaiEtAl2019,DidionKomjathySutinEtAl2018}, and swarms of small spacecraft to study small bodies \cite{stacey2018autonomous}. In all of these applications, the heterogeneous task-allocation problem is central: a set of robots with heterogeneous sensing capabilities is tasked with observing a variety of scientific events of interest (e.g. dust storms, volcanism, or changes in a body's surface); a dynamic model that approximately predicts the spatio-temporal distribution and evolution of the phenomenon of interest is available; and, while certain tasks (e.g., radio science or medium-resolution imaging) can be performed by all agents, other specialized tasks (e.g. hyperspectral imaging, deployment of sondes, or sampling) can only be performed by a subset of the robots.


\ifextendedv
\subsection{Related work}
\else
\fakeparagraph{Related work.}
\fi
Task allocation is an enabling subroutine for applications like closely-coupled coordination (e.g., deciding which robot takes which place in a formation) and for loosely-coupled coordination (e.g., which robot traverses to a new spot to observe an interesting target). However, most variants of the problem are known to be computationally prohibitive to solve~\cite{GerkeyMataric04,KorashETAL13}. Thus, many approaches for task allocation 
scale poorly with the problem size (e.g., the number of robots) or provide no guarantees on the solution quality or runtime. Furthermore, heterogeneous robot capabilities (e.g., ground vehicles and aerial drones jointly working to achieve a mutual goal), impose additional challenges for designing practical, high-quality solution approaches~\cite{BaiETAL20,AgatzETAL18,FerrandezETAL16,MurrayChu15,Wang2017}. 

Several algorithm types have been proposed specifically for 
task allocation by the robotics community (see e.g. the survey in~\cite{RossiBandyopadhyayEtAl2018}). In {auction-based algorithms} \cite{GerkeyMataric04,Ayanian17,KoenigEA10}, robots bid on tasks based on their state and capabilities. Auction-based algorithms can be readily implemented in a distributed fashion and naturally accommodate heterogeneous robots. Spatial partitioning algorithms \cite{PavoneFrazzoliEtAl2011} rely on partitioning the workspace into regions and assigning each region to one or multiple robots. Tasks within a region are assigned to the robot (or robots) responsible for that region. Spatial partitioning algorithms capture the likelihood of occurrence of future events. 
Team-forming and temporal partitioning algorithms \cite{Smith09} group heterogeneous robots in teams so that each team is capable of performing all the tasks that might arise. 
Mixed-integer linear programming (MILP) approaches~\cite{Bellingham03} explicitly represent the task-allocation problem as a mixed-integer program, and can readily capture a variety of constraints including heterogeneity. 
Markov chain-based algorithms \cite{Bandyopadhyay17} model the robots' motion via a stochastic policy prescribed by a Markov chain optimized according to a given cost function. 
While all those approaches cover a wide range of problems and techniques, they generally either do not scale well with the problem size or provide weak theoretical guarantees on the solution quality or runtime.

\ifextendedv
The predictive nature of our problem strongly relates to the notion of online algorithms~\cite{FiatWoeginger98,BorodinElYaniv05,HentenryckBent09}, which are designed for problems in which the input is  revealed gradually, while optimizing a goal function (e.g., minimizing cost). The typical benchmark for online algorithms is the {worst-case} ratio between the solution of the online algorithm and the optimal solution for the {offline} case in which the entire input is given a priori. This is termed as the algorithm's {competitive ratio}. The {$k$-server problem}~\cite{Koutsoupias09,BertsimasETAL19}, which has been studied in this online context, bears some resemblance to our setting. In particular, it can be viewed as the centralized case consisting of a single fleet of $k>1$ homogenenous agents, where the goal is to collect all the rewards while minimizing travel cost in an online fashion. A recent paper introduced an online randomized algorithm for the problem which achieves a competitive ratio of $O(\log^6 k)$~\cite{Lee18}. The weighted variant of the problem, which is closer to our heterogeneous setting, requires an even larger competitive ratio of $\Omega(2^{2^{k-4}})$~\cite{BansalETAL17}. 
\fi

Applications of homogeneous task allocation have been extensively explored recently within the setting of transportation and logistics. E.g., the operation of an autonomous mobility-on-demand system requires to assign ground vehicles to routes in order to fulfill passenger demand, while potentially accounting for road congestion~\cite{SoloveyETAL19,WallerETAL18,Levine17}. Recent work develops efficient package delivery framework consisting of multiple drones in which drones are assigned to packages and delivery routes~\cite{ChoudhurySoloveyETAL2020}, and proposes utilizing public-transit vehicles on which drones can hitchhike in order to conserve their limited energy,  thus noticeably increasing their service range. From a broader perspective, task allocation can  be viewed through the lens of the vehicle routing problem (VRP)~\cite{TothVigo2014} or the orienteering problem (OP)~\cite{GunawanETAL16}. 
However, save a few special cases, VRP and OP  are typically approached with MILP formulations that scale poorly, or heuristics that do not provide optimality guarantees.

The problem that we address in this paper can be represented as a multi-agent pathfinding (MAPF) problem~\cite{YuLaValle16}. The goal in MAPF is to compute a  collection of paths for the agents to minimize execution time, while accounting for inter-agent conflict constraints. 
Unfortunately, 
no polynomial-time approximation algorithms that can solve general MAPF instances exist~\cite{yu2013structure}, to the best of our knowledge. 
A recent approach for MAPF termed conflict-based search~\cite{felner2017search,ma2017lifelong,honig2018conflict,liu2019task} earned popularity due to its efficiency in moderately-sized instances. However, it does not provide run-time guarantees and does not scale well in settings that require considerable amount of coordination among agents. 

\ifextendedv
\subsection{Contribution}
\else
\fakeparagraph{Contribution.}
\fi
We present an efficient approximation algorithm, termed \alg, for heterogeneous task-allocation in the context of maximizing the collection of time-varying rewards.
From the theoretical perspective we prove that our algorithm achieves an approximation factor of at least $1/2$ of the optimal reward in polynomial time. The algorithm also exhibits good performance in practice. Specifically, in simulation experiments, we demonstrate that the algorithm achieves a speedup of several orders of magnitude over a MILP approach, and its runtime scales modestly with the problem size. Moreover, we prove that the runtime is insensitive to the number of agents present in each fleet. 

From an algorithmic standpoint, the \alg algorithm decomposes the heterogeneous problem into several homogeneous subproblems,  that can be solved efficiently using min-cost flow~\cite{Williamson2019}, without significant degradation in solution quality. Our approach shares some similarity with a recent work that also considers homogeneous decomposition~\cite{PrasadETAL20}, albeit  the problem setting (computing multiple travelling-salesman routes in an undirected and time-invariant graph), as well as the algorithmic techniques and analysis that are developed there, are quite different from ours.

\ifextendedv
\subsection{Organization}
\else
\fakeparagraph{Organization.}
\fi 
The organization of this paper is as follows. In Section~\ref{sec:problem_statement} we provide basic definitions and the problem formulation. In Section~\ref{sec:homogeneous} we first study the homogeneous subproblem. 
Then, in Section~\ref{sec:centralized} we describe the \alg algorithm and provide its theoretical analysis.
In Section~\ref{sec:experiments} we provide experimental results demonstrating the good performance and scalability of our approach. We conclude with a discussion of future research directions in Section~\ref{sec:future}.

\begin{figure}[t]
    \centering
    \includegraphics[width=1\columnwidth,trim={4cm 1.3cm 1cm 3cm},clip]{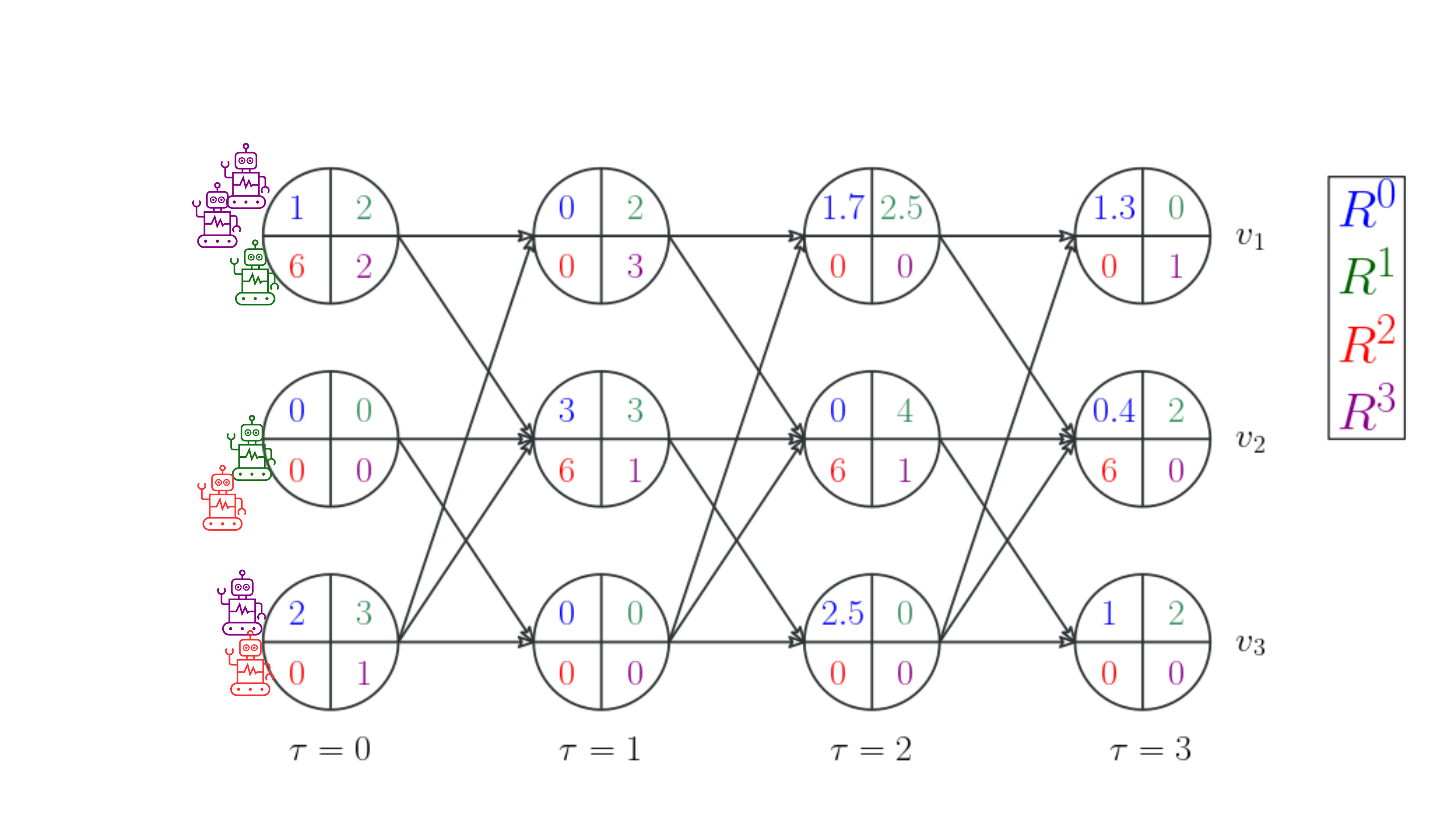}
    \caption{Illustration of the problem setting. In this example, the workspace $\G$, which consists of three vertices $v_1, v_2, v_3$ and seven edges (going from, e.g., the vertices at time $\tau=0$ to $\tau=1$), is expanded over the time horizon $T=3$. There are $F=3$ fleets of heterogeneous agents, where $a_1=2$ (green), $a_2=2$ (red), $a_3=3$ (magenta). For every vertex $v\in \V$ and time step $\tau$, the values in the corresponding time-expanded vertex represent the following rewards: shared reward $R^0_\tau[v]$ (blue), and private reward $R^1_\tau[v], R^2_\tau[v], R^3_\tau[v]$ for fleets $1$ (green), $2$ (red), and $3$ (magenta), respectively. The initial positions are $p_0^1=(1,1,0),p_0^2=(0,1,1), p_0^3=(2,0,1)$. \label{fig:workspace}}
    \vspace{-15pt}
\end{figure}

\section{Preliminaries and problem formulation \label{sec:problem_statement}}

We first describe the problem ingredients and then proceed to a formal definition of the problem.
%
%
The robots' workspace is represented by a directed graph $\G=(\V,\E)$, with vertices $i\in \V$ denoting \frmargin{physical locations}{Really, vertices are for tasks, physical locations are a special case. Do we want to explain or are we happy with the current formulation?}\ksmargin{-}{I'm fine with the current text, though I agree that it can be made more precise.} for the robots, and edges $(i,j)\in \E$ denoting transitions between locations (see example in Figure~\ref{fig:workspace}). We use $\E^+_i$ to denote the set of outgoing neighbors of a vertex $i\in \V$, namely $\{j\in \V|(i,j)\in \E\}$. We similarly define $\E^-_i$ to represent the set of incoming neighbors, i.e., $\{j\in \V|(j,i)\in \E\}$.

We consider a discrete-time, finite-horizon framework, where the horizon is specified by a positive integer $T$. We use $\tau\in [0\isep T]$ to denote a given time step, 
where  $[r\isep r']$ denotes an integer interval $\{k\in \dN:r\leq k \leq r'\}$ between two integers such that $r<r'$. 


Throughout this paper, we refer to the robots as ``robots" or ``agents" interchangeably. The set of all agents is denoted by $\A$, which is subdivided into $F$ disjoint fleets $\A^1,\ldots, \A^F$, where agents within the same fleet are assumed to have homogeneous capabilities.
We denote by $a_f\!=\!|\A^f|$ the number of agents in a fleet $f\!\in\! \F$, where  $\F:=[1\isep F]$ is the set of fleet indices. 
The agents are mobile and transition from one vertex $i\!\in\! \V$ to another $j\!\in\! \V$ every time step, assuming that $(i,j)\!\in\! \E$. 
For a fleet $f\!\in\! \F$, the positions of its agents at time $\tau\!\in\! [0\isep T]$ are specified by the vector $p^f_\tau\!\in\! [0\isep a_f]^{ |\V|}$, the value $p^f_\tau[i]$ specifies the number of agents of $\A^f$ located at a vertex $i\!\in\! \V$. Figure~\ref{fig:workspace} provides a graphical representation of the problem formulation. 

\subsection{Shared and private rewards}
The problem consists of allocating agents to rewards along the time-expanded vertices of $\G$. For now, we assume that the distributions of rewards are known in advance. We discuss a predictive extension where the reward sets are not known in advance in Section~\ref{sec:predictive}.

There are $F+1$ types of reward sets $R^0,R^1,\ldots,R^F$, which determine the values gained by the agents for visiting any given vertex $i\in \V$ at time $\tau\in [0\isep T]$, and there are constraints that specify which fleets $f\in \F$ can collect a specific reward $R^t$ of type  $t\in [0\isep F]$. Specifically, rewards of type $t=0$ are considered to be \emph{shared}, i.e., can be collected by any robot of any fleet $f\in \F$. In particular, for a given vertex $i\in \V$ and time step $\tau\in [0\isep T]$, the system gains the reward $R^0_\tau[i_\tau]$ if $\sum_{f\in \F}p^f_\tau[i]\geq 1$, i.e.,  at least one agent (from any fleet) visits the vertex $i$ at time $\tau$.
In contrast, rewards of type $t \neq 0$ are considered to be \emph{private}, and can only be gained via  agents from the particular fleet $\A^f$ such that $f=t$. I.e., when $t\neq 0$ then for any $i\in \V,\tau\in [0\isep T]$, the system gains the reward $R^t_\tau[i]$ if $p^t_\tau[i]\geq 1$, i.e.,  at least one agent from $\A^t$ visits the vertex $i$ at time $\tau$. 

\begin{table}[t]
    \noindent\fbox{
    \begin{minipage}[t]{0.48\textwidth - 2\fboxsep - 2\fboxrule}%
            Given the graph $\G$, time horizon $T$, agent fleets $\A^1,\ldots,\A^F$, initial positions $p_0$, and reward sets $R^0,\ldots, R^F$, the objective is to \emph{maximize}
      \begin{subequations}
      \begin{align}
          \R(x,y,z):= \sum_{j\in \V}\! \sum_{\tau\in [0\isep T]} \! \sum_{f\in \F}\left(R^0_\tau[j]\cdot y^f_\tau[j] +  R^f_\tau[j]\cdot z^f_\tau[j] \right) \label{eq:objective_het}
      \end{align}
        \begin{align}
          \underset{\begin{array}{ll}
x_{\tau}^{f}[i,j]\in [0\isep a_f], & \forall (i,j) \in \E, f\in \F,  \tau \in [0\isep T-1], \\
y_\tau^f[j]\in \{0,1\}, z_\tau^f[j]\in \{0,1\}, & \forall j\in \V,   f\in \F, \tau \in [0\isep T],\\
\end{array}}{\textrm{with the decision variables}  \quad \quad \quad \quad \quad \quad \quad \quad \quad \quad \quad \quad \quad \quad } \nonumber
        \end{align}
        \begin{align}
          \textrm{subject to} \nonumber \\
          \sum_{i\in \E^+_j}\!\! x^f_0[j,i] & = p^f_0[j] \thinspace, & & \!\! \!\!  \forall j \in \V,  f \in \F, \label{eq:init_position} \\
          \sum_{i\in \E^-_j}\!\!\! x^f_{\tau-1}[i,j] & = \!\!\sum_{\ell\in \E^+_j}\!\!\! x^f_\tau[j,\ell], & & \!\! \!\! \forall j \in \V,  f \in \F, \tau\!\in\! [1\isep T\!-\!2], \label{eq:flow_conservation} \\
           \sum_{i\in \E^+_j}\!\! x^f_\tau[j,i] & \geq y^f_\tau[j], & & \!\! \!\!  \forall j \in \V, f\in\F, \tau\!\in\! [0\!\isep\! T-1], \label{eq:bound_y_1} \\ 
          \sum_{i\in \E^-_j}\!\! x^f_{T-1}[i,j] & \geq  y^f_T[j] , & & \!\! \!\! \forall j \in \V,  f\in\F,  \label{eq:bound_y_2} \\ 
          \sum_{f\in\F}\!\! y^f_\tau[j] & \leq 1 , & & \!\! \!\!  \forall j\in \V,  \forall\tau\in [0\isep T],\label{eq:bound_y_3} \\ 
          \sum_{i\in \E^+_j}\!\! x^f_\tau[j,i] & \geq z^f_\tau[j]  , & & \!\! \!\! \forall j \in \V,   f\!\in\!\F,  \tau\!\in\! [0\!\isep\! T-1], \label{eq:bound_z_1} \\ 
           \sum_{i\in \E^-_j}\!\! x^f_{T-1}[i,j]  & \geq z^f_T[j] , & & \!\! \!\! \forall  j \in \V, f\in\F,  \label{eq:bound_z_2} \\ 
          z^f_\tau[j] & \leq 1 , & &  \!\! \!\!  \forall j\in \V,  f\in \F ,\tau\in [0\isep T]. \label{eq:bound_z_3}
        \end{align}
      \end{subequations}
    \end{minipage}}
\caption{Definition of the heterogeneous task-allocation problem. \label{tbl:heterogeneous}}
\vspace{-25pt}
\end{table}

\subsection{Problem formulation}
We provide a formal definition of our problem in the form of an integer program (Table~\ref{tbl:heterogeneous}). The input to the problem consists of the workspace graph $\mathcal{G}$, time horizon $T\in \dN_{>0}$, agent fleets $\A^1,\ldots,\A^F$ with $F\in \dN_{>0}$ with known initial positions $p^f_0[i]$ for all $i\in V$, and rewards sets $R^0,\ldots, R^F$. The goal of this work is to obtain a task-allocation scheme, which maximizes the total collected reward. The task-allocation scheme consists of (i) specifying the locations of all agents for every time step $\tau \in [0\isep T]$ and (ii) assigning agents to rewards.

The solution is described through two types of decision variables. The integer variable $x^f_\tau[i,j]\in [0\isep a_f]$ denotes a transition of agents in fleet $f\in\F$ from vertex $i\in \V$ to $j\in \V$, at time $\tau\in [0\isep T-1]$, assuming that $(i,j)\in \E$.  
The decision variables $y^f_\tau[i]\in \{0,1\}$ and $z^f_\tau[i]\in \{0,1\}$ indicate whether an agent from fleet $f\in \F$ is assigned to collect the shared reward $R^0_\tau[i]$, or private reward  $R^f_\tau[i]$, respectively. 

The objective function is given in Eq.~\eqref{eq:objective_het}. Eq.~\eqref{eq:init_position} ensures that agents will start at their initial positions as specified by $p_0$. Eq.~\eqref{eq:flow_conservation} ensures the continuity of the agents. 
Eq.~\eqref{eq:bound_y_1}, \eqref{eq:bound_y_2}, ensure that an agent from fleet $f$ is assigned to a shared reward  $R^0_\tau[j]$ only if one of the agents of this fleet is at vertex $j$ in time $\tau$; similarly, Eq.~\eqref{eq:bound_z_1} and \eqref{eq:bound_z_2} enforce this condition with respect to private reward sets. Eq.~\eqref{eq:bound_y_3}, \eqref{eq:bound_z_3} limit the number of agents assigned to every reward type in a given vertex to~$1$. 

We mention that our problem formulation and solution approach can be easily extended to accommodate edge traversal costs by modifying the objective functions in Tables~\ref{tbl:heterogeneous} and~\ref{tbl:homogeneous}. For simplicity of presentation we do not account for edge costs in the current formulation.

\subsection{Predictive setting}\label{sec:predictive}
The above formulation can be extended to the predictive setting, where we are given the values of the reward sets for the first time step $R_0^0,\ldots,R_0^f$, and a stochastic model of the evolution of rewards with respect to time. To exploit the aforementioned deterministic formulation, it is straightforward to show that by plugging into the problem defined in Table~\ref{tbl:heterogeneous} the expected values of the stochastic rewards, the solution would maximize the expected gained reward.

This gives rise to a receding-horizon implementation: given the current state of the system (e.g., locations of agents and values of current rewards) and the current time step $\tau$ we predict the reward values $R^0,\ldots,R^F$ for $T$ time steps into the future, and  compute a corresponding solution $x^\tau,y^\tau,z^\tau$. We then execute this solution for the first time step, obtain current values of the reward, and repeat this process in the next time step $\tau+1$. 
For the simplicity of presentation, we shall focus on the static setting of the problem from now on. We do note that our experiments are for the predictive case. 





\subsection{Discussion}
A few comments are in order.
First, the discrete nature of the graph representation is well-suited to the discrete, combinatorial nature of the task allocation problem, where agents must be assigned to \emph{discrete} tasks. The formulation is highly flexible: for example, vertices can represent tasks to be performed, events to be observed, or regions to be surveyed, with edges representing the travel time between these events. In this context, both the formulation in Table~\ref{tbl:heterogeneous} and our \alg algorithm can be straightforwardly extended to a setting where the goal is to maximize the reward minus the travel cost, where the latter is represented by costs assigned to the edges of $\G$.

The approach is also amenable to settings where events and tasks appear in a continuous workspace, via appropriate discretization. However, in such cases a careful and problem-specific selection of the appropriate degree of discretization is critical to ensure good performance.

Our approach does not account for collision avoidance between the agents.
In typical autonomy frameworks (e.g.,~\cite{VolpeNesnasea01, Wolf17}), team-level decisions such as task allocation are handled at a different, higher level compared to collision avoidance, which is often resolved through local controllers (e.g.,~\cite{ma2017lifelong,honig2018conflict,liu2019task,ChoudhurySoloveyETAL2020}). 
Unless the workspace is heavily cluttered with obstacles, such a hierarchical architecture typically results in good performance, while offering greatly-reduced complexity compared to a hypothetical  integrated architecture.

\section{Algorithm for the homogeneous case}\label{sec:homogeneous}
In preparation to the \alg algorithm for the heterogeneous problem, we describe  a key ingredient, which is an efficient solution to the homogeneous problem. The homogeneous problem consists of maximizing the collected reward for a single fleet $f\in \F$ of homogeneous agents and a private reward set $R^f$. Our main insight is that an optimal solution to the homogeneous problem can be found efficiently by solving a min-cost flow (MCF) problem~\cite{Williamson2019}.

\begin{table}[!ht]
\noindent\fbox{\begin{minipage}[0\isep T]{0.48\textwidth - 2\fboxsep - 2\fboxrule}%
Given the graph $\G$, time horizon $T$, agent fleet $A^f$, initial positions $p^f_0$, and private reward set $R^f$,
      \begin{subequations}
        \begin{align*}
          \textrm{maximize} \quad  \R^f(x,z):= \sum_{j\in \V} \sum_{\tau\in [0\isep T]}  R^f_\tau[j]\cdot z^f_\tau[j], 
        \end{align*}
\begin{align}
 \textrm{subject to  (\ref{eq:init_position}), (\ref{eq:flow_conservation}), (\ref{eq:bound_z_1}), (\ref{eq:bound_z_2}), (\ref{eq:bound_z_3}), } \text{ with respect to } \A^f. \nonumber 
\end{align}
\end{subequations}
\end{minipage}}
\caption{Definition of the heterogeneous task-allocation problem.\label{tbl:homogeneous}}
\vspace{-15pt}
\end{table}

Denote by $\H(R^f, p_0^f)$ the homogenenous optimization problem described in Table~\ref{tbl:homogeneous}, for a private reward set  $R^f$ and initial positions $p_0^f$ of a fleet $f\in \F$. 
Note that the problem of assigning the set of shared rewards $R^0$ to all the agents $\A$, while ignoring the assignment of private rewards, can be viewed as the homogeneous problem $\H(R^0,p)$. The following lemma states that an optimal solution for the homogeneous problem can be obtained in 
polynomial time.

\begin{lemma}[Efficient solution of $\H$]\label{lem:homogeneous}
  For a given fleet $f\in \F$, private reward set $R^f$, and initial positions $p_0^f$, the optimal solution for the homogeneous problem $\H(R^f, p^f_0)$  can be computed in $\O\left(T^2mn\log (Tn)+T^2n^2\log (Tn)\right)$
  time, where $m=|\E|,n=|\V|$. This bound also holds for the homogeneous problem  $\H(R^0,p_0)$ with the shared reward $R^0$.
\end{lemma}

\ifextendedv
\begin{proof}
We show that the homogeneous problem can be transformed into a min-cost flow (MCF) problem~\cite{Williamson2019}, for which an optimal solution can be efficiently found. In the remainder of this proof we define the ingredients for MCF, formally define the problem, describe its relation to the homogeneous problem, and finally discuss its complexity.

MCF is defined over a directed graph $\dG=(\dV,\dE)$, which has to designated vertices $s,g\in \dV$, that represent the source and sink, respectively. In addition, for every edge $e\in\dE$ we have the attributes $u_e>0$ and $c_e\in \dR$ which represent the edge's capacity and traversal cost, respectively. We also have a variable $h^*$ which denotes the total flow that emerges from the source $s$ and needs to arrive to  the sink $g$. 

The goal in MCF is to assign flow values $h_e\in [0,\infty)$, which represent the number of agents traversing $e$, to all edges $e\in \dE$ in order to minimize the expression $\sum_{e\in \dE}h_e c_e$ under the constraints 
\begin{subequations}
\begin{align}
 h_{e}&\in [0,u_e],  \forall e\in \dE,\label{eq:mcf:bound}\\
 \sum_{v\in \dE^+_s}h_{sv}& = \sum_{v\in \dE^-_g}h_{vg}= h^*, \label{eq:mcf:source}\\
  \sum_{v'\in \dE^+_v}h_{vv'}-\sum_{v'\in \dE^-_v}h_{v'v}&= 0, \forall v\in \dV\setminus \{s,g\}.\label{eq:mcf:else}
\end{align}
\end{subequations}
Constraint \eqref{eq:mcf:bound} ensures that edge capacities are maintained; Constraint~\eqref{eq:mcf:source} ensures that the flow leaving $s$ and entering $g$ would be equal to $h^*$; Constraint~\eqref{eq:mcf:else} ensures flow conservation.

For a given instance of $\H(R^f,p^f)$ we construct an equivalent instance of the MCF problem where flows correspond to agent locations and their reward assignments. In particular, we define the graph $\dG=(\dV,\dE)$, such that 
$$\dV=\{s,g\}\cup\{v^i_\tau,w^i_\tau|i\in \V,\tau\in [0\isep T]\}.$$
That is, the vertex set consists of the source $s$ and sink $g$ vertices, as well as two copies $v^i_\tau,w^i_\tau$ of every vertex $i\in \V$ from the workspace graph $\G$ for each time step $\tau\in [0\isep T]$. 

The edge set of $\dG$ is a union of several edge sets, i.e., $\dE=E_{s}\cup E_g\cup E_0\cup E_{R}\cup E_1$, where
\begin{align*}
  E_s&=\{(s,v^i_0)|\exists q\in \A^f, p^q_0[i]\neq 0\},\\
  E_g&=\{(w^i_T,g)|i\in \V\},\\
  E_0&=\{(v^i_\tau,w^i_\tau)|i\in \V, \tau\in [0\isep T]\},\\
  E_{R}&=\{(v^i_\tau,w^i_\tau)|i\in \V, \tau\in [0\isep T],R^f_\tau[i]\neq 0\},\\
  E_{\E}&=\{(w^i_\tau,v^j_{\tau+1})|(i,j)\in \E, \tau\in [0\isep T-1]\}.  
\end{align*}
Namely, edges in $E_s$ connect the source vertex $s$ to all the start vertices of the agents. Edges in $E_g$ connect all the vertices in the final time step $T$ to the sink node $g$. Edges in $E_0$ connect every two copies $v^i_\tau,w^i_\tau$ of a vertex $i\in V$ for a given time step $\tau$. Agents traversing the latter set of edges will receive no reward. On the other hand, every edge $(v^i_\tau,w^i_\tau)\in E_R$ will be used to represent the collection of the reward $R^f_\tau[i]$ by the agent traversing this edge. Edges in $E_{\E}$ simulate the traversal of agents along the edges of $\E$ between one time step to the next. 

Next, we assign costs for the edges in $\dE$. In particular, for any $e\in \dE\setminus E_R$ we set $c_e:=0$. For a given $e'\in(v^i_\tau,w^i_\tau)\in E_R$ we set $c_{e'}:=-R^f_\tau[i]$. The final ingredient to the MCF problem is represented by the edge capacities $u$. For any $e\in E_g\cup E_0\cup E_\E$ we set the infinite capacity $u_e:=\infty$. To ensure that the correct number of agents will arrive at the starting positions, we specify for a given $e'=(s,v_0^i)\in E_s$ the capacity $u_{e'}:=\sum_{q\in \A^f}p_0^q[i]$. Finally, to ensure that every reward will be collected by at most one agent, we specify $u_{e''}:=1$ for any $e''\in E_R$. 

By definition the above MCF formulation is equivalent to the homogeneous problem, with one slight change. Namely, MCF permits fractional assignments to the flow variables $h$, whereas the values of $x,z$ are integral. However, if all edge capacities are integral (which is the case in our setting), the linear relaxation of MCF enjoys a totally-unimodular constraint matrix form~\cite{AhujaETAL93}. Hence, the fractional solution will necessarily have an integer optimal solution. 

The MCF problem above can be solved in $\O((M'+N)\log N(M+N\log N))$ time using Orlin's MCF algorthim~\cite{Orlin93,Williamson2019}, where $M=|\dE|=\O(Tm)$, $N=|\dV|=\O(Tn)$, and $M'$ is the number of $\dG$ edges whose capacity is finite. In our case $M'=\O(Tn)$, which implies that the total time complexity for solving the above MCF formulation is $\O(T^2mn\log n+T^2n^2\log^2n)$. 
\end{proof}
\else
To prove this result we show that the homogeneous problem $\H(R^f, p^f_0)$ can be equivalently represented as MCF. Given a graph $\dG=(\dV,\dE)$, edge capacities $u_e$ and costs $c_e$ for every edge $e\in\dE$, the objective in MCF is to assign flow values $h_e$ to each edge $e$,  to minimize the total flow cost $\sum_{e\in \dE}h_e\cdot c_e$, while satisfying edge capacity constraints $u_e$. 

In the full proof of Lemma~\ref{lem:homogeneous}, which we defer to the extended version of the paper~\cite{Solovey.Bandyopadhyay.ea.21}, we define the MCF ingredients $\dG,u,c$ such that the 
constraints of the resulting MCF problem correspond to the constraints of the homogeneous problem $\H(R^f, p^f_0)$, and the flow variables $h$ describe the agents' locations $x^f$ and reward assignments $z^f$. Moreover, an optimal solution to this MCF problem yields an optimal solution to $\H(R^f, p^f_0)$. The runtime complexity bound in Lemma~\ref{lem:homogeneous} follows from using Orlin's algorithm for MCF~\cite{Orlin93}.
\fi

\section{Algorithm for the heterogeneous case}\label{sec:centralized}
We present an efficient algorithm, which we call \alg, for the heterogeneous task-allocation problem described in Table~\ref{tbl:heterogeneous}. The \alg algorithm decomposes the problem into several homogeneous subproblems that are solved using min-cost flow, as described in Section~\ref{sec:homogeneous}. In the remainder of this section we describe \alg, and determine its approximation and runtime guarantees.


\subsection{The \alg algorithm}
The \alg algorithm (Algorithm~\ref{alg:main}) accepts as input the agent fleets $\A^1,\ldots, \A^F$, initial positions $p_0$, and reward sets $\dR:=\{R^0,R^1,\ldots,R^F\}$. 
Recall that we wish to find an assignment $x,y,z$ such that the expression $\R(x,y,z)$ is maximized. Also recall that for a given fleet $f\in \F$, timestep $\tau \in [0\isep T]$, and vertices $i,j\in V$, $x_{\tau}^f[i,j]$ represents the transitions of the agents in $\A^f$ from $i$ to $j$, and $y_{\tau}^f[i]$, $z_{\tau}^f[i]$ indicate whether those agents are assigned to collect the rewards $R_{\tau}^0[i]$ and $R_{\tau}^f[i]$, respectively, at vertex $i$. For a given fleet $f\in \F$, denote by $x^f$ the corresponding values of the $x$ assignment for agents belonging to fleet $f$, i.e., $x^f=\{x^f_{\tau}\}_{\tau\in [0\isep T-1]}$. The sets $y^f,z^f$ are similarly defined. 

\alg computes two candidate solutions 
using the subroutines \private and \shared, respectively, and returns the one that yields the larger reward of the two. Next we elaborate on those two subroutines. 

\begin{algorithm}[!h]
  $(x,y,z)\gets \private(\dR,\A,p_0)$\;
  $(x',y',z')\gets \shared(\dR,\A,p_0)$\;
  \If {$\R(x,y,z)>\R(x',y',z')$}{
    \Return $(x,y,z)$\;
    }
  \Return $(x',y',z')$\;
  \caption{\alg$(\dR,\A=\{\A^1,\ldots,\A^F\},p_0)$}
  \label{alg:main} 
\end{algorithm} 

The subroutine \private (Algorithm~\ref{alg:private}) prioritizes the assignment of private rewards over shared rewards. This is achieved by assigning to each fleet $f\in \F$ a new reward set $\hat{R}^f$ that combines the private reward set $R^f$ and the shared reward set $R^0$, where the latter is rescaled by $F^{-1}$, i.e., $\hat{R}^f_\tau[j]={R}^f_\tau[j]+ {R}^0_\tau[j]\cdot F^{-1}$. An assignment over $\hat{R}^f$ for every $f\in \F$ is then obtained by solving the homogeneous problem $\H(\hat{R}^f,p_0^f)$. Note that $z^f$ implicitly encodes both an assignment to a private reward and shared reward. That is, an agent assigned to perform a reward $\hat{R}^f_\tau[j]$ can be interpreted as being assigned to both $R^f_\tau[j]$ and $R^0_\tau[j]$.
In lines 5-9 the solutions of the individual fleets are combined to eliminate cases where several agents (from different fleets) are assigned to the same shared reward. To do so, for a given time step $\tau$ and vertex $j$, we iterate over all fleets $f\in \F$ and assign $y^f_\tau[j]=1$ for the first agent we encounter that is assigned to a shared reward in the corresponding vertex. 

The \shared subroutine (Algorithm~\ref{alg:shared}) prioritizes the assignment of shared rewards, by first computing an assignment for all the agents $\A$ to the shared reward set $R^0$, to maximize the total reward. This is achieved by solving the homogeneous problem $\H(R^0,p_0)$.  It then generates an updated private reward set $\bar{R}^f$ for every fleet $f\in \F$, where the value of a reward $\bar{R}^f_\tau [j]$ is equal to $R^f_\tau [j]$ in case that $R^0_\tau[j]$ was not assigned to $f$, and otherwise equal to $R^f_\tau [j]+R^0_\tau[j]$, for every time step $\tau$ and vertex $j$. Next, for every fleet~$f$ the private assignment $(x^f,z^f)$ over $\bar{R}^f$ is computed by solving $\H(\bar{R}^f,p^f_0)$. Note that $z$ here represents simultaneously assignments for shared and private rewards.

\begin{algorithm}[!ht]
  $\hat{R}^f\gets R^f+R^0\cdot F^{-1}, \forall f\in \F$\;
  $(x^f,z^f)\gets \H(\hat{R}^f,p_0^f), \forall f\in \F$\;
  $x\gets \{x^f\}_{f\in \F},z\gets\{z^f\}_{f\in \F}$\;
  $y\gets \{\{0\}_{q\in \A}\}_{\tau \in [0\isep T]}$\;
  \For{$\tau\in [0\isep T], j\in \V$}{
    \For{$f\in \F$}{
        \If{$z_\tau^f[j]==1$}{
            $y_\tau^f[j]\gets 1$\;
            break \;
            }
        }
    }
  \Return $(x,y,z)$\;
  \caption{$\private(\dR,\A,p_0)$}
  \label{alg:private}
\end{algorithm} 

\begin{algorithm}
    $(\bar{x},\bar{y})\gets \H(R^0,p_0)$\; 
    \For{$f\in \F, \tau\in [0\isep T], j\in \V$}{
        $\bar{R}^f_\tau[j]\gets {R}^f_\tau[j]+R^0_\tau[j]\cdot\bar{y}_\tau^f[j]$\;
    }
  $(x^f,z^f)\gets \H(\bar{R}^f,p^f_0), \forall f\in \F$\;
  $x\gets \{x^f\}_{f\in \F},z\gets\{z^f\}_{f\in \F}$\;
  \Return $(x,z,z)$\;
  \caption{$\shared(\dR,\A,p_0)$}
  \label{alg:shared}
\end{algorithm} 
\vspace{-10pt}

\subsection{Analysis of \alg}
We prove that the \alg algorithm is guaranteed to achieve a solution within a constant factor of the optimum. Let $(x,y,z)$ be a solution of \alg. We use $\R(x^f,y^f,z^f)$ to represent the portion of the total reward $\R(x,y,z)$ that is attributed to fleet $f\in \F$ (where assignment values of agents in other fleets are set to $0$). Similarly, denote by $\R(x,y,\bm{0})$ and $\R(x,\bm{0},z)$ the shared and private portion of the total reward, respectively, where~$\bm{0}$ represents the zero vector (whose dimension will be clear from context). 

Let $(X,Y,Z)$ be a solution to the heterogeneous problem (Table~\ref{tbl:heterogeneous}) that maximizes the expression $\R(X,Y,Z)$ and define
$\textsc{opt}:=\R(X,Y,Z)=S^*+P^*$, where $S^*:=\R(X,Y,\bm{0})$ and $P^*:=\R(X,\bm{0},Z)=\sum_{f\in \F}\R({X}^f,\bm{0},{Z}^f)$.
We are ready to state our main theoretical contribution:

\begin{theorem}[Approximation factor of \alg]\label{thm:main}
  Let $(x,y,z)$ be the solution returned by \alg. Then $\R(x,y,z)\geq \textsc{opt}\cdot \frac{F}{2F-1}$. 
\end{theorem}

Before proceeding to the proof 
we establish two intermediate results, concerning the guarantees of \private and \shared, when considered separately. In particular, we show that each of the subroutines provide complementary approximations with respect to $S^*$ and $P^*$, and \alg enjoys the best of both worlds. 
\ifextendedv
See illustration in Figure~\ref{fig:approximation}
\fi

\ifextendedv
\begin{figure}[!ht]
    \centering
    \includegraphics[width=2.5in]{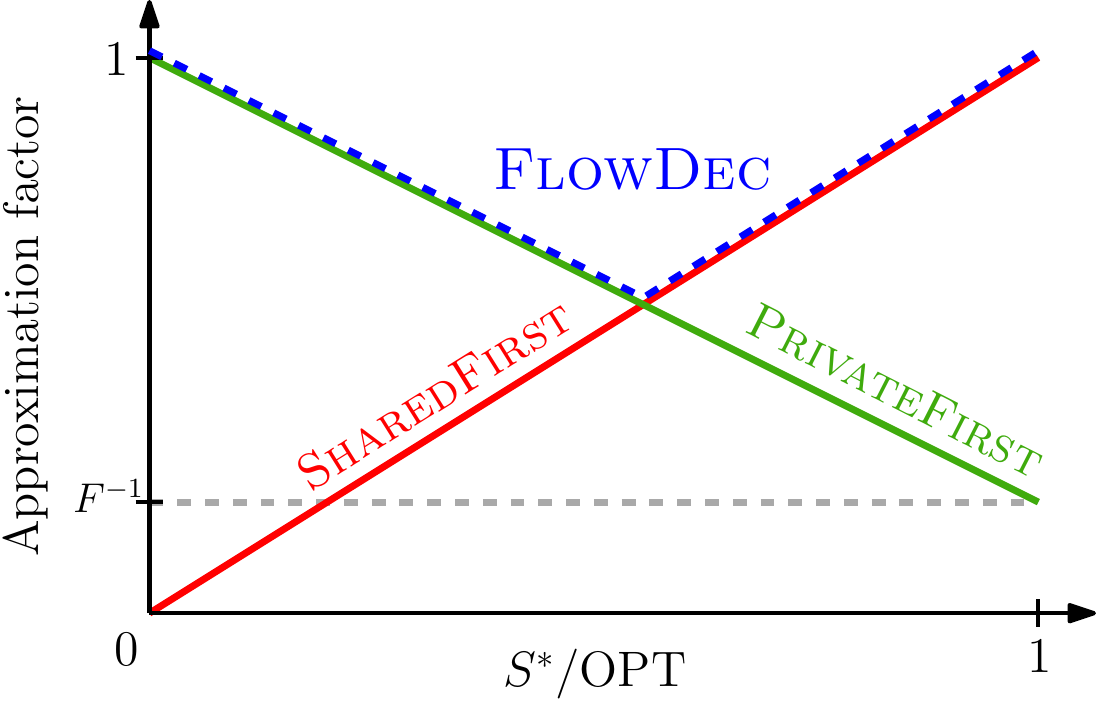}
    \caption{Visualization of the approximation factors (y-axis) achieved by \alg (dashed blue curve), \private (green curve), \shared (red curve), as a function of the ratio between the shared value of the optimal assignment $S^*$ and OPT.}
    \label{fig:approximation}
\end{figure}
\else
\fi

\begin{claim}[\private's solution quality]\label{clm:private}
  If $(x,y,z)$ is the solution returned by \private, then $\R(x,y,z)\geq P^*+ F^{-1}\cdot S^*$.
\end{claim}
\begin{proof}
  Fix a fleet $f\in \F$ and note that
  \[\R(x^f,y^f,z^f) \geq \R\left({X}^f,\bm{0},{Z}^f\right)+F^{-1}\cdot \R(X^f,Y^f,\bm{0}),\]
  since \private is free to choose $(x^f,z^f):=(X^f,Y^f+Z^f)$ (line~2 in Algorithm~\ref{alg:private}). Hence,
\ifextendedv
\else
{\small
\fi
\begin{align*}
  \R(x,y,z)&=\sum_{f\in \F}\R(x^f,y^f,z^f)\\ & \geq  \sum_{f\in \F}\R\left({X}^f,\bm{0},{Z}^f\right)+F^{-1}\sum_{f\in \F}\R (X^f,Y^f,\bm{0}) \\ 
& =P^* + F^{-1}\cdot S^*.\qedhere
\end{align*}
\ifextendedv
\else
}
\fi
\end{proof}

\begin{claim}[\shared's solution quality]\label{clm:shared}
  If $(x,y,z)$ is the solution returned by \shared then $\R(x,y,z)\!\geq\! S^*$.
\end{claim}
\begin{proof}
  The proof follows from the inequality $\R(\bar{x},\bar{y},\bm{0})\geq \R(X,Y,\bm{0})$, where $\bar{x},\bar{y}$ are defined Algorithm~\ref{alg:shared}, line~1.
\end{proof}

We are ready for the main proof:
\begin{proof}[Proof of Theorem~\ref{thm:main}]
Claims~\ref{clm:private} and~\ref{clm:shared} imply that \alg obtains a solution $(x,y,z)$ with reward at least $\max\left\{P^*+F^{-1}\cdot S^*,S^*\right\}$. If either of $S^*$ or $P^*$ is zero, then \alg would return the optimal solution. Thus, assume instead that $S^*$ and $P^*$ are positive. Thus, there exists $\Delta>0$ such that $P^*=\Delta S^*$. 
The approximation factor of \alg can be expressed as follows:
\begin{align*}
    \frac{\R(x,y,z)}{\R(X,Y,Z)} & \geq \frac{\max\left\{\Delta\cdot S^*+F^{-1}\cdot S^*,S^*\right\}}{\Delta S^*+S^*} \\ & 
    = \max\left\{\frac{\Delta+F^{-1}}{\Delta +1},\frac{1}{\Delta+1}\right\}\\ 
    & \geq 
    \max\left\{\frac{\frac{F-1}{F}+F^{-1}}{\frac{F-1}{F} +1},\frac{1}{\frac{F-1}{F}+1}\right\} = \frac{F}{2F-1},
\end{align*}
where the last inequality follows from $\argmin_{\Delta>0}\max\left\{\frac{\Delta+F^{-1}}{\Delta +1},\frac{1}{\Delta+1}\right\}=(F-1)/F$, which is used since we wish to find a lower-bound that is applicable to any combination of $P^*$ and $S^*$.
\end{proof}

We conclude with a runtime analysis of \alg.

\begin{corollary}[\alg runtime]\label{cor:runtime}
  \alg can be implemented in $\O(FT^2mn\log n+FT^2n^2\log^2n)$ time.
\end{corollary}
\begin{proof}
  The bottleneck of \alg is solving multiple homogeneous subproblems. Since \private and \shared solve $2F+1$ homogeneous problems, respectively, and each such computation requires $\O(T^2mn\log n+T^2n^2\log^2n)$ time (Lemma~\ref{lem:homogeneous}), the total runtime follows. 
\end{proof}

\section{Experimental results}\label{sec:experiments}
We validate our theoretical results from the previous section through simulation experiments. We show that our \alg approach is faster than a MILP approach by several orders of magnitude, and we observe that the approximation factors that \alg achieves in practice are higher than the worst-case lower bound of $\frac{F}{2F-1}$. Additionally, 
we observe experimentally that the running time of \alg is insensitive to the number of agents in each fleet.

\subsection{Implementation and scenario details}
The results were obtained using a laptop with 2.80GHz 4-core i7-7600U CPU, and 16GB of RAM. We implemented the \alg algorithm in \cpp, using network simplex for MCF in the LEMON Graph Library ~\cite{Lemon,Kovacs15}. Our code is available at {\small \url{github.com/StanfordASL/heterogeneous-task-allocation
}}. For comparison, we used the the MILP implementation in CPLEX~\cite{Cplex} with the default stopping criterion using a relative-gap and absolute tolerances of 1e-4 and 1e-6, respectively.

We tested both implementations on a predictive problem formulation, as described in Section~\ref{sec:predictive}. We consider the problem of heterogeneous tracking of multiple moving objects, where the reward values are chosen to incentivise agents to visit graph vertices where objects are located.  In particular, for a given graph $\G$, time horizon $T$, number of fleets $F$, and initial object count $I\in \dN_+$, we chose uniformly at random for each reward type $t\in [0\isep F]$, $I$ vertices of $\G$ (with repetitions) which represent initial locations of objects to be tracked as part of the reward set $R^t$. In particular, we set the value of $R^t_0[i]$ to be the number of objects at a given vertex $i\in \V$. For the subsequent time steps $\tau\in [1\isep T]$ we set $R^t_\tau$ to be the expected reward assuming that the objects move to neighboring vertices via a random walk. Initial agent locations are chosen in a uniform random fashion. 

\subsection{Results}
We compare the solution quality and runtime of the MILP approach and our \alg algorithm. We then study the scalability of the \alg algorithm on larger test cases for which the MILP approach has timed out. 

\newcolumntype{L}{>{\centering\arraybackslash}m{0.3cm}}
\newcommand{\dnf}{-}

\begin{table}[!t]
\vspace{5pt}
\centering
\begin{tabular}{Lc||c|c|c|c|c|c|c|}
\cline{3-9}
                                           &      & \multicolumn{7}{c|}{fleets}    \\ \hline
\multicolumn{1}{|L||}{time hor.}         & type & 2 & 4 & 8 & 16 & 32 & 64 & 128 \\ \hline \hline
\multicolumn{1}{|c||}{\multirow{3}{*}{2}}   & MILP & 0.00 & 0.01 & 0.01 & 0.03 & 0.07 & 0.16 & 0.36 \\ \cline{2-9} 
\multicolumn{1}{|c||}{}                     & \textsc{Flow} & 0.01 & 0.01 & 0.02 & 0.03 & 0.06 & 0.12 & 0.26   \\ \cline{2-9} 
\multicolumn{1}{|c||}{}                     & APX  & 1.00 & 1.00 & 0.93 & 0.86 & 0.64 & 0.75 & 0.75   \\ \hline  \hline
\multicolumn{1}{|c||}{\multirow{3}{*}{4}}   & MILP &  0.02 & 0.04 & 0.07 & 0.16 & 0.39 & 0.90 & 1.77  \\ \cline{2-9} 
\multicolumn{1}{|c||}{}                     & \textsc{Flow} & 0.01 & 0.03 & 0.05 & 0.09 & 0.17 & 0.35 & 0.67   \\ \cline{2-9} 
\multicolumn{1}{|c||}{}                     & APX  &  1.00 & 0.97 & 0.87 & 0.86 & 0.77 & 0.74 & 0.85    \\ \hline  \hline
\multicolumn{1}{|c||}{\multirow{3}{*}{8}}   & MILP &   0.11 & 0.26 & 0.58 & 1.39 & 3.62 & 12 & 21  \\ \cline{2-9} 
\multicolumn{1}{|c||}{}                     & \textsc{Flow} &  0.04 & 0.07 & 0.13 & 0.25 & 0.46 & 0.94 & 1.85   \\ \cline{2-9} 
\multicolumn{1}{|c||}{}                     & APX  &  0.96 & 0.92 & 0.82 & 0.80 & 0.84 & 0.88 & 0.94     \\ \hline  \hline
\multicolumn{1}{|c||}{\multirow{3}{*}{16}}  & MILP &  0.45 & 1.40 & 3.78 & 20 & 93 & 453 & \dnf   \\ \cline{2-9} 
\multicolumn{1}{|c||}{}                     & \textsc{Flow} &  0.10 & 0.19 & 0.36 & 0.67 & 1.27 & 2.54 & 5.92   \\ \cline{2-9} 
\multicolumn{1}{|c||}{}                     & APX  & 0.96 & 0.93 & 0.88 & 0.92 & 0.89 & 0.95 &     \\ \hline  \hline
\multicolumn{1}{|c||}{\multirow{3}{*}{32}}  & MILP &  1.5 & 5.6 & 20 & 284 & 567 & \dnf & \dnf \\ \cline{2-9} 
\multicolumn{1}{|c||}{}                     & \textsc{Flow} &  0.29 & 0.52 & 1.00 & 1.87 & 3.76 & 7.11 & 16  \\ \cline{2-9} 
\multicolumn{1}{|c||}{}                     & APX  &  0.94 & 0.89 & 0.82 & 0.88 & 0.94 &  &   \\ \hline \hline
\multicolumn{1}{|c||}{\multirow{3}{*}{64}}  & MILP &  4.78 & 21 & 203 & \dnf & \dnf & \dnf & \dnf  \\ \cline{2-9} 
\multicolumn{1}{|c||}{}                     & \textsc{Flow} &  0.84 & 1.50 & 2.78 & 5.22 & 12 & 20 & 43 \\ \cline{2-9} 
\multicolumn{1}{|c||}{}                     & APX  &  0.97 & 0.91 & 0.92 &  &  &  &      \\ \hline  \hline
\multicolumn{1}{|c||}{\multirow{3}{*}{128}} & MILP & 14.20 & \dnf & \dnf & \dnf & \dnf & \dnf & \dnf    \\ \cline{2-9} 
\multicolumn{1}{|c||}{}                     & \textsc{Flow} &  2.48 & 4.85 & 8.25 & 17 & 31 & 59 & 115  \\ \cline{2-9} 
\multicolumn{1}{|c||}{}                     & APX  &  0.96 &  &  &  &  &  &   \\ \hline  
\end{tabular}
\caption{Comparison between \alg and a MILP approach in terms of runtime and solution quality for $10\times 10$ grid graphs. For every combination of number of fleets $F$ and time horizon $T$ we report in the ``MILP'' and ``\textsc{Flow}'' rows the corresponding running times (in seconds). The label ``\dnf'' indicates that MILP did not finish within the 10-minute time limit. In the ``APX'' row we report the approximation factor of \alg, i.e., the quotient between the reward values obtained by \alg and MILP, respectively.\label{tbl:milp}} 
\end{table}

\subsubsection{Comparison between \alg and \textsc{MILP}}
In this setup, we fix the graph $\G$ to be a $10\times 10$ grid, set the initial number of tracked objects for every reward to be $I=3$, and set the number of agents within each fleet $f\in \F$ to be $a_f=5$. In Table~\ref{tbl:milp} we report the running time of the MILP solution and the \alg algorithm, as well as the approximation factor that was achieved by \alg, for scenarios of varying sizes. We set the time horizon $T$ and the fleet number $F$ to be in the range $[2\isep 128]$. The reported running times are averaged over $20$ randomly-generated scenarios for each parameter combination.  The reported approximation factor is the minimum (i.e., worst) result over the 20 runs. We terminate the run of each algorithm if it exceeds $10$ minutes, in which case we do not its running  time or the approximation factor.

In terms of running time, we observe that the  MILP approach behaves similarly to \alg only for the smallest test cases, e.g., when $T\in \{2,4\}$. However, as the problem size increases the running time of the MILP approach grows significantly faster than that of \alg. For instance, already for $T=8$, when $F=2$ \alg is nearly $3$ times faster than the MILP approach, and when $F=128$ it is more than $10$ times faster. As $T$ is increased we encounter more scenarios in which the MILP implementation is forced to time out, whereas \alg finishes fairly quickly. For example, when $T=16$, \alg finishes within a few seconds, for all fleet sizes, whereas the MILP approach times out for $F=128$, which yields a speedup of at least $100$ for \alg. For $T=128$,  the MILP approach is able to solve only the smallest scenario, whereas \alg solves all of them.
In terms of solution quality, we observe that \alg typically achieves approximation factors that are larger than the theoretical 
lower bound, which suggests that this bound is loose. We note that the smallest approximation factor that we observed with \alg is $0.64$ (for $F=32$).

\subsubsection{Scalability of \alg} 
We rerun the previous experiment for a larger $50\times 50$ graph to test how the runtime of the \alg algorithm is affected by its different parameters. We report in Table~\ref{tbl:pfsf} the runtime of the algorithm, where $I=3, a_f=5$ were set as in the previous experiment. In accordance with the theoretical complexity bound in Corollary~\ref{cor:runtime}, we observe that the runtime increases linearly with the number of fleets $F$. The runtime increases linearly with the time horizon $T$ as well, which suggests that the theoretical quadratic increase is overly conservative. 

Finally, we note that the runtime of \alg is only mildly affected by the size of each individual fleet, e.g.,  as we increase the value of $a_f$ from $5$ to $500$, we observe a modest increase of $\%10$ to the runtime. This is in contrast to the MILP approach which is highly sensitive to this value. 

\begin{table}[]
\vspace{10pt}
\begin{center}
\begin{tabular}{c|c|c|c|c|c|c|c|}
\cline{2-8}
                                   & \multicolumn{7}{c|}{fleets}    \\ \hline
\multicolumn{1}{|L||}{time hor.} & 2 & 4 & 8 & 16 & 32 & 64 & 128 \\ \hline \hline
\multicolumn{1}{|c||}{2}            & 0.1 & 0.2 & 0.4 & 0.8 & 1.4 & 4.8 & 7.3  \\ \hline
\multicolumn{1}{|c||}{4}            & 0.3 & 0.6 & 1.1 & 2.2 & 4.2 & 13 & 19  \\ \hline
\multicolumn{1}{|c||}{8}            &  1.0 & 1.6 & 3.2 & 6.6 & 12 & 38 & 55  \\ \hline
\multicolumn{1}{|c||}{16}           &  3 & 5 & 9 & 19 & 35 & 118 & 214   \\ \hline
\multicolumn{1}{|c||}{32}           &   9 & 14 & 29 & 59 & 132 & 214 & 414  \\ \hline
\multicolumn{1}{|c||}{64}           &   31 & 48 & 100 & 189 & 375 & 700 & 1423  \\ \hline
\multicolumn{1}{|c||}{128}          &  74 & 154 & 299 & 518 & 994 & 2040 & 4055   \\ \hline
\end{tabular}
\vspace{-5pt}
\end{center}
\caption{Running time (seconds) of \alg for a $50\times 50$ grid as a function of the number of fleets $F$ and time horizon $T$. \label{tbl:pfsf}}
\end{table}

\section{Conclusion}\label{sec:future}
We presented a near-optimal algorithm, termed \alg, for heterogeneous task allocation and demonstrated its good performance through extensive simulation tests. Our work suggests a few interesting directions for future research, which we highlight below. 
From an algorithmic perspective we plan to explore whether the approximation factor of \alg can be improved by introducing a third subroutine which would better estimate the optimal reward for scenarios in which the subroutines \shared and \private under 
\ifextendedv
approximate (see Figure~\ref{fig:approximation}).
\else
approximate.
\fi
It would also be interesting to test whether a specialized MCF solver for object tracking~\cite{WangEA19} can speed up the solution of homogeneous problems.   
We also plan to consider extending our approach to account for collision-avoidance constraints by exploiting the fact that those constraints can be captured via a MCF-based solution for the homogeneous problem. Finally, we  plan to extend our algorithm to a distributed setting, by potentially relying on a dual decomposition that would be applied to the homogeneous subproblems ~\cite{Ref:Bertsekas2016Nonlinear}.

{\small
\section*{Acknowledgments}
Part of this research was carried out at the Jet Propulsion Laboratory, California Institute of Technology, under a contract with the National Aeronautics and Space Administration (80NM0018D0004). This work was supported in part by the Toyota Research Institute (TRI) and the Center for Automotive Research at Stanford (CARS). 
\ifextendedv
The authors thank Xiaoshan Bai, Robin Brown, Shushman Choudhury, Devansh Jalota, Erez Karpas, Javier Alonso-Mora, and Matthew Tsao for fruitful discussions.
\fi
}
{
\bibliographystyle{unsrt}
\bibliography{main}

\begin{thebibliography}{10}

\bibitem{lillis2020mars}
RJ~Lillis, DL~Mitchell, L~Montabone, S~Guzewich, SM~Curry, P~Withers,
  MS~Chaffin, TN~Harrison, CO~Ao, NG~Heavens, et~al.
\newblock Mars orbiters for surface-atmosphere-ionosphere connections
  ({MOSAIC}).
\newblock {\em LPI}, (2326):1733, 2020.

\bibitem{KrishnamoortyKomjathyEtAl2020}
Siddharth Krishnamoorthy, Attila Komjathy, James~A. Cutts, Philippe Lognonne,
  Raphael~F. Garcia, Mark~P. Panning, Paul~K. Byrne, Robin~S. Matoza, Art~D.
  Jolly, Jonathan~B. Snively, Sebastien Lebonnois, and Daniel Bowman.
\newblock Seismology on {V}enus with infrasound observations from balloon and
  orbit.
\newblock White Paper for the NASA 2021 Decadal Survey SAND2020-2849R 684580,
  Sandia National Lab, Albuquerque, NM, 3 2020.

\bibitem{KrishnamoorthyLaiEtAl2019}
S.~{Krishnamoorthy}, V.~H. {Lai}, A.~{Komjathy}, M.~T. {Pauken}, J.~A. {Cutts},
  R.~F. {Garcia}, D.~{Mimoun}, J.~M. {Jackson}, D.~C. {Bowman}, E.~{Kassarian},
  L.~{Martire}, A.~{Sournac}, and A.~{Cadu}.
\newblock Aerial seismology using balloon-based barometers.
\newblock {\em IEEE Transactions on Geoscience and Remote Sensing},
  57(12):10191--10201, 2019.

\bibitem{DidionKomjathySutinEtAl2018}
A.~{Didion}, A.~{Komjathy}, B.~{Sutin}, B.~{Nakazono}, A.~{Karp}, M.~{Wallace},
  G.~{Lantoine}, S.~{Krishnamoorthy}, M.~{Rud}, J.~{Cutts}, P.~{Lognonné},
  B.~{Kenda}, M.~{Drilleau}, J.~{Makela}, M.~{Grawe}, and J.~{Helbert}.
\newblock Remote sensing of {v}enusian seismic activity with a small
  spacecraft, the {VAMOS} mission concept.
\newblock In {\em IEEE Aerospace Conference}, pages 1--14, 2018.

\bibitem{stacey2018autonomous}
Nathan Stacey and Simone D’Amico.
\newblock Autonomous swarming for simultaneous navigation and asteroid
  characterization.
\newblock In {\em AAS/AIAA Astrodynamics Specialist Conference}, 2018.

\bibitem{GerkeyMataric04}
Brian~P Gerkey and Maja~J Matari{\'c}.
\newblock A formal analysis and taxonomy of task allocation in multi-robot
  systems.
\newblock {\em International Journal of Robotics Research}, 23(9):939--954,
  2004.

\bibitem{KorashETAL13}
G~Ayorkor Korsah, Anthony Stentz, and M~Bernardine Dias.
\newblock A comprehensive taxonomy for multi-robot task allocation.
\newblock {\em International Journal of Robotics Research}, 32(12):1495--1512,
  2013.

\bibitem{BaiETAL20}
X.~{Bai}, M.~{Cao}, W.~{Yan}, and S.~S. {Ge}.
\newblock Efficient routing for precedence-constrained package delivery for
  heterogeneous vehicles.
\newblock {\em IEEE Transactions on Automation Science and Engineering},
  17(1):248--260, 2020.

\bibitem{AgatzETAL18}
Niels Agatz, Paul Bouman, and Marie Schmidt.
\newblock Optimization approaches for the traveling salesman problem with
  drone.
\newblock {\em Transportation Science}, 52(4):965--981, 2018.

\bibitem{FerrandezETAL16}
Sergio~Mourelo Ferrandez, Timothy Harbison, Troy Weber, Robert Sturges, and
  Robert Rich.
\newblock Optimization of a truck-drone in tandem delivery network using
  k-means and genetic algorithm.
\newblock {\em Journal of Industrial Engineering and Management},
  9(2):374--388, 2016.

\bibitem{MurrayChu15}
Chase~C. Murray and Amanda~G. Chu.
\newblock The flying sidekick traveling salesman problem: {O}ptimization of
  drone-assisted parcel delivery.
\newblock {\em Transportation Research Part C: Emerging Technologies}, 54:86 --
  109, 2015.

\bibitem{Wang2017}
Xingyin Wang, Stefan Poikonen, and Bruce Golden.
\newblock The vehicle routing problem with drones: {S}everal worst-case
  results.
\newblock {\em Optimization Letters}, 11(4):679--697, Apr 2017.

\bibitem{RossiBandyopadhyayEtAl2018}
F.~Rossi, S.~Bandyopadhyay, M.~Wolf, and M.~Pavone.
\newblock Review of multi-agent algorithms for collective behavior: a
  structural taxonomy.
\newblock In {\em {IFAC Workshop on Networked \& Autonomous Air \& Space
  Systems}}, 2018.
\newblock In Press.

\bibitem{Ayanian17}
Nora Ayanian.
\newblock {DART:} diversity-enhanced autonomy in robot teams.
\newblock {\em International Journal of Robotics Research}, 38(12-13), 2019.

\bibitem{KoenigEA10}
Sven Koenig, Pinar Keskinocak, and Craig~A. Tovey.
\newblock Progress on agent coordination with cooperative auctions.
\newblock In {\em Conference on Artificial Intelligence}. {AAAI} Press, 2010.

\bibitem{PavoneFrazzoliEtAl2011}
M.~Pavone, E.~Frazzoli, and F.~Bullo.
\newblock Adaptive and distributed algorithms for vehicle routing in a
  stochastic and dynamic environment.
\newblock {\em {IEEE Transactions on Automatic Control}}, 56(6):1259--1274,
  2011.

\bibitem{Smith09}
Stephen~L Smith and Francesco Bullo.
\newblock The dynamic team forming problem: Throughput and delay for unbiased
  policies.
\newblock {\em Systems \& control letters}, 58(10-11):709--715, 2009.

\bibitem{Bellingham03}
John Bellingham, Michael Tillerson, Arthur Richards, and Jonathan~P How.
\newblock Multi-task allocation and path planning for cooperating {UAVs}.
\newblock In {\em Cooperative control: {M}odels, applications and algorithms},
  pages 23--41. Springer, 2003.

\bibitem{Bandyopadhyay17}
S.~Bandyopadhyay, S.-J. Chung, and F.~Y. Hadaegh.
\newblock Probabilistic and distributed control of a large-scale swarm of
  autonomous agents.
\newblock {\em {IEEE} Trans on Robotics}, 33(3):1103--1123, 2017.

\bibitem{FiatWoeginger98}
Amos Fiat and Gerhard~J Woeginger.
\newblock {\em Online algorithms: The state of the art}, volume 1442.
\newblock Springer, 1998.

\bibitem{BorodinElYaniv05}
Allan Borodin and Ran El-Yaniv.
\newblock {\em Online computation and competitive analysis}.
\newblock Cambridge University Press, 2005.

\bibitem{HentenryckBent09}
Pascal~Van Hentenryck and Russell Bent.
\newblock {\em Online stochastic combinatorial optimization}.
\newblock The MIT Press, 2009.

\bibitem{Koutsoupias09}
Elias Koutsoupias.
\newblock The $k$-server problem.
\newblock {\em Computer Science Review}, 3(2):105 -- 118, 2009.

\bibitem{BertsimasETAL19}
Dimitris Bertsimas, Patrick Jaillet, and Nikita Korolko.
\newblock The $k$-server problem via a modern optimization lens.
\newblock {\em European Journal of Operational Research}, 276(1):65 -- 78,
  2019.

\bibitem{Lee18}
James~R. Lee.
\newblock Fusible {HST}s and the randomized $k$-server conjecture.
\newblock In {\em Foundations of Computer Science}, pages 438--449, 2018.

\bibitem{BansalETAL17}
N.~{Bansal}, M.~{Eliáš}, and G.~{Koumoutsos}.
\newblock Weighted $k$-server bounds via combinatorial dichotomies.
\newblock In {\em Foundations of Computer Science}, pages 493--504, 2017.

\bibitem{SoloveyETAL19}
Kiril Solovey, Mauro Salazar, and Marco Pavone.
\newblock Scalable and {C}ongestion-{A}ware {R}outing for {A}utonomous
  {M}obility-on-{D}emand via {F}rank-{W}olfe {O}ptimization.
\newblock In {\em Robotics: Science and Systems}, 2019.

\bibitem{WallerETAL18}
Alex Wallar, Menno Van~Der Zee, Javier Alonso{-}Mora, and Daniela Rus.
\newblock Vehicle {R}ebalancing for {M}obility-on-{D}emand {S}ystems with
  {R}ide-{S}haring.
\newblock In {\em {IEEE/RSJ} International Conference on Intelligent Robots and
  Systems}, pages 4539--4546, 2018.

\bibitem{Levine17}
Michael~W. Levin.
\newblock Congestion-aware system optimal route choice for shared autonomous
  vehicles.
\newblock {\em Transportation Research Part C: Emerging Technologies}, 82:229
  -- 247, 2017.

\bibitem{ChoudhurySoloveyETAL2020}
Shushman Choudhury, Kiril Solovey, Mykel Kochenderfer, and Marco Pavone.
\newblock Efficient {L}arge-{S}cale {M}ulti-{D}rone {D}elivery using {T}ransit
  {N}etworks.
\newblock In {\em International Conference on Robotics and Automation}, 2020.

\bibitem{TothVigo2014}
P.~Toth and D.~Vigo.
\newblock {\em Vehicle routing -- {P}roblems, methods, and applications}.
\newblock SIAM, 2 edition, 2014.

\bibitem{GunawanETAL16}
Aldy Gunawan, Hoong~Chuin Lau, and Pieter Vansteenwegen.
\newblock Orienteering problem: A survey of recent variants, solution
  approaches and applications.
\newblock {\em European Journal of Operational Research}, 255(2):315 -- 332,
  2016.

\bibitem{YuLaValle16}
J.~{Yu} and S.~M. {LaValle}.
\newblock Optimal multirobot path planning on graphs: {C}omplete algorithms and
  effective heuristics.
\newblock {\em {IEEE} Transactions on Robotics}, 32(5):1163--1177, 2016.

\bibitem{yu2013structure}
Jingjin Yu and Steven~M LaValle.
\newblock Structure and intractability of optimal multi-robot path planning on
  graphs.
\newblock In {\em {AAAI} Conference on Artificial Intelligence}, 2013.

\bibitem{felner2017search}
Ariel Felner, Roni Stern, Solomon~Eyal Shimony, Eli Boyarski, Meir Goldenberg,
  Guni Sharon, Nathan Sturtevant, Glenn Wagner, and Pavel Surynek.
\newblock Search-based optimal solvers for the multi-agent pathfinding problem:
  {S}ummary and challenges.
\newblock In {\em Symposium on Combinatorial Search}, 2017.

\bibitem{ma2017lifelong}
Hang Ma, Jiaoyang Li, TK~Kumar, and Sven Koenig.
\newblock Lifelong multi-agent path finding for online pickup and delivery
  tasks.
\newblock In {\em International Conference on Autonomous Agents and Multiagent
  Systems}, pages 837--845, 2017.

\bibitem{honig2018conflict}
Wolfgang H{\"o}nig, Scott Kiesel, Andrew Tinka, Joseph~W Durham, and Nora
  Ayanian.
\newblock Conflict-based search with optimal task assignment.
\newblock In {\em International Conference on Autonomous Agents and Multiagent
  Systems}, pages 757--765, 2018.

\bibitem{liu2019task}
Minghua Liu, Hang Ma, Jiaoyang Li, and Sven Koenig.
\newblock Task and path planning for multi-agent pickup and delivery.
\newblock In {\em International Conference on Autonomous Agents and Multiagent
  Systems}, pages 1152--1160, 2019.

\bibitem{Williamson2019}
David~P. Williamson.
\newblock {\em Network Flow Algorithms}.
\newblock Cambridge University Press, 2019.

\bibitem{PrasadETAL20}
A.~{Prasad}, H.~{Choi}, and S.~{Sundaram}.
\newblock Min-max tours and paths for task allocation to heterogeneous agents.
\newblock {\em {IEEE} {T}ransactions on Control of Network Systems}, 2020.

\bibitem{VolpeNesnasea01}
Richard Volpe, Issa Nesnas, Tara Estlin, Darren Mutz, Richard Petras, and Hari
  Das.
\newblock The {CLARAty} architecture for robotic autonomy.
\newblock In {\em IEEE Aerospace Conference Proceedings}, volume~1, pages
  1/121--1/132 vol.1, 2001.

\bibitem{Wolf17}
Michael~T. Wolf, Amir Rahmani, Jean-Pierre de~la Croix, Gail Woodward, Joshua
  Vander~Hook, David Brown, Steve Schaffer, Christopher Lim, Philip Bailey,
  Scott Tepsuporn, Marc Pomerantz, Viet Nguyen, Cristina Sorice, and Michael
  Sandoval.
\newblock {CARACaS} multi-agent maritime autonomy for unmanned surface vehicles
  in the {Swarm II} harbor patrol demonstration.
\newblock In {\em Unmanned Systems Technology}, volume 10195, pages 1--11,
  2017.

\bibitem{AhujaETAL93}
Ravindra~K. Ahuja, Thomas~L. Magnanti, and James~B. Orlin.
\newblock {\em Network flows: Theory, algorithms, and applications}.
\newblock Pearson, 1993.

\bibitem{Orlin93}
James~B. Orlin.
\newblock A faster strongly polynomial minimum cost flow algorithm.
\newblock {\em Operations Research}, 41(2):338--350, 1993.

\bibitem{Lemon}
{LEMON Graph Library}.
\newblock \url{https://lemon.cs.elte.hu/trac/lemon}.
\newblock Accessed: 07.2020.

\bibitem{Kovacs15}
P{\'e}ter Kov{\'a}cs.
\newblock Minimum-cost flow algorithms: an experimental evaluation.
\newblock {\em Optimization Methods and Software}, 30(1):94--127, 2015.

\bibitem{Cplex}
{IBM}.
\newblock {ILOG CPLEX} optimization studio, 2020.

\bibitem{WangEA19}
Congchao Wang, Yizhi Wang, Yinxue Wang, Chiung{-}Ting Wu, and Guoqiang Yu.
\newblock {muSSP}: {E}fficient min-cost flow algorithm for multi-object
  tracking.
\newblock In {\em Neural Information Processing Systems}, pages 423--432, 2019.

\bibitem{Ref:Bertsekas2016Nonlinear}
Dimitri~P Bertsekas.
\newblock {\em Nonlinear programming}.
\newblock Athena Scientific, Belmont, USA, 3rd edition, 2016.

\end{thebibliography}
}

\end{document}